\documentclass{article}
\usepackage{tikz}
\usepackage{microtype}
\usepackage{graphicx}
\usepackage{subfigure}
\usepackage{booktabs} 
\usepackage{hyperref} 
\usepackage{tikz}
\usepackage{subfigure}
\usepackage{subfloat}
\usepackage{listings}
\usepackage{algorithm}
\usepackage{algorithmic} 
\usepackage{appendix}
\usepackage{graphicx} 
\usepackage{natbib}
\usepackage{wrapfig}
\usepackage{color} 
\usepackage{amsmath,amsthm,amsfonts,amssymb,bm}
\usepackage[utf8]{inputenc} 
\usepackage[T1]{fontenc}    
\usepackage{hyperref}       
\usepackage{url}            
\usepackage{booktabs}       
\usepackage{amsfonts}       
\usepackage{nicefrac}       
\usepackage{microtype}      
\usepackage{dsfont}
\usepackage{booktabs}
\usepackage{bbm}

\newtheorem{Thm}{Theorem}
\newtheorem{Lem}{Lemma}
\newtheorem{Cor}{Corollary}
\newtheorem{Asmp}{Assumption}
\newtheorem{Defn}{Definition}

\newtheorem{Prop}{Property}

\DeclareMathOperator*{\argmax}{arg\,max}

\usepackage[accepted]{icml2019}

\icmltitlerunning{Lazy-CFR: fast and near-optimal regret minimization for extensive games with imperfect information}

\definecolor{codegreen}{rgb}{0,0.6,0}
\definecolor{codegray}{rgb}{0.5,0.5,0.5}
\definecolor{codepurple}{rgb}{0.58,0,0.82}
\definecolor{backcolour}{rgb}{0.95,0.95,0.92}

\lstdefinestyle{mystyle}{
    backgroundcolor=\color{backcolour},   
    commentstyle=\color{codegreen},
    keywordstyle=\color{magenta},
    numberstyle=\tiny\color{codegray},
    stringstyle=\color{codepurple},
    basicstyle=\footnotesize,
    breakatwhitespace=false,         
    breaklines=true,                 
    captionpos=b,                    
    keepspaces=true,                 
    numbers=left,                    
    numbersep=5pt,                  
    showspaces=false,                
    showstringspaces=false,
    showtabs=false,                  
    tabsize=2
}

\begin{document}

\twocolumn[
\icmltitle{Lazy-CFR: fast and near-optimal regret minimization for extensive games with imperfect information}

\begin{icmlauthorlist}
\icmlauthor{Yichi Zhou}{thu}
\icmlauthor{Tongzheng Ren}{thu}
\icmlauthor{Dong Yan}{thu}
\icmlauthor{Jialian Li}{thu}
\icmlauthor{Jun Zhu}{thu}
\end{icmlauthorlist}

\icmlaffiliation{thu}{Dept. of Comp. Sci. \& Tech., BNRist Center, State Key Lab for Intell. Tech. \& Sys., THBI Lab, Tsinghua University, Beijing, 100084, China}
\icmlcorrespondingauthor{Jun Zhu}{dcszj@mail.tsinghua.edu.cn}
\icmlcorrespondingauthor{Yichi Zhou}{vofhqn@gmail.com}

\icmlkeywords{Machine Learning, ICML}

\vskip 0.3in
]



\printAffiliationsAndNotice{}  

\begin{abstract}
Counterfactual regret minimization (CFR) is the most popular algorithm on solving two-player zero-sum extensive games with imperfect information and 
achieves state-of-the-art results in practice.
However, the performance of CFR is not fully understood, since empirical results on the regret are much better than the known upper bound in \cite{zinkevich2008regret}. Moreover, CFR has to traverse the whole game tree in each round, which is time-consuming in large scale games
. In this paper, we present a novel technique, lazy update, which can avoid traversing the whole game tree in each round. We propose a novel analysis on the regret of CFR with lazy update, which can also be applied to the vanilla CFR, resulting in a much tighter regret bound than that in \cite{zinkevich2008regret}. Inspired by lazy update, we further present a novel CFR variant, named Lazy-CFR. Compared to traversing $O(|\mathcal{I}|)$ information sets in the vanilla CFR, Lazy-CFR needs only to traverse $O(\sqrt{|\mathcal{I}|})$ information sets per round while keeping the regret bound almost the same, where $\mathcal{I}$ is the class of all information sets. As a result, Lazy-CFR shows better convergence results compared with the vanilla CFR. Experimental results consistently show that Lazy-CFR outperforms the vanilla CFR significantly.
\end{abstract}

\section{Introduction}\label{Sec:intro}
Extensive games provide a mathematical framework for modeling the sequential decision-making problems with imperfect information, 
which is common in economic decisions, negotiations and security. In this paper, we focus on solving two-player zero-sum extensive games with imperfect information (TEGI). In a TEGI, there are an environment with uncertainty and two players on opposite sides \citep{koller1992complexity}.  

Counterfactual regret minimization (CFR) \citep{zinkevich2008regret} provides a state-of-the-art algorithm for solving TEGIs with much progress in practice \citep{brown2017safe,moravvcik2017deepstack}. 
The most famous application of CFR is Libratus, the first program that defeats top human players in heads-up no-limit Texas Hold’em poker 
\citep{brown2017safe}. CFR works based on the fact that minimizing the regrets of both players makes the time-averaged strategy to approach the Nash Equilibrium (NE) \citep{zinkevich2008regret}.  
Furthermore,  
CFR bounds the original regret with a summation of many immediate counterfactual regrets, each of which corresponds to an infomation set (infoset).  
These immediate regrets are defined by counterfactual rewards and they can be iteratively minimized by online learning algorithms, e.g., regret matching (RM) \citep{blackwell1956analog} and Hedge \citep{freund1997decision}.  

Though CFR has succeeded in practice, the behavior of CFR is not fully understood. Specifically, experiments have shown that the regret is significantly smaller than the upper bound in \citep{zinkevich2008regret}. So at least  some further theoretical analysis can be provided on the regret bound of CFR. 
Besides, a more crucial limitation of CFR is that it requires traversing the whole game tree in each round, which is time-consuming  
in large-scale games. This is because we have to apply RM to every immediate regret in each round.  
Though various attempts have been made to avoid traversing the whole game tree in each round so that they can significantly speed up the vanilla CFR in practice, they are lack of theoretical guarantees on the running time or can even degenerate in the worst case~\citep{brown2015regret,brown2016reduced,lanctot2009monte}.

In this paper, we present a novel technique called lazy update, which provides a unified framework to avoid traversing the whole game tree in CFR. For each infoset, 
CFR with lazy update segments the time horizon into disjoint subsets with consecutive elements. We call these subsets as segments.  
And then CFR with lazy update updates the strategy only at the start of each segment and  keeps the strategy on that infoset the same within each segment. So that lazy update can save computation resources. It is noteworthy that our framework includes the vanilla CFR as a degenerated case, in which the length of each segment is $1$.  
Moreover, we present a novel  
analysis on the regret of CFR with lazy update. Our analysis is also based on the immediate regrets as in \citep{zinkevich2008regret}. The difference is that, in contrast to  \citep{zinkevich2008regret}'s analysis which takes each immediate regret independently, our analysis reveals the correlation among them via the underlying optimal strategy. Specifically, we prove that it is impossible that immediate regrets are all very large simultaneously. 
As an application of our analysis, we refine the regret bound 
of CFR from $O(|\mathcal{I}|\sqrt{T\log A})$ to $O(\sqrt{\xi DT\log A})$ \footnote{The upper bound in \citep{zinkevich2008regret} is $O(|\mathcal{I}|\sqrt{T A})$, since they use RM as the online learning solver. The regret bound can be reduced into $O(|\mathcal{I}|\sqrt{T\log A})$ by simply replacing RM with Hedge. } where $\mathcal{I}=\cup_{i}\mathcal{I}^i$, $\mathcal{I}^i$ is the infosets of player $i$, $A$ is the number of actions, $D$ is the depth of the game tree, $T$ is the  length of time and $\xi$ is a quantity reflecting the structure of the game tree, whose value varies from $D$ to $|\mathcal{I}|$, and in most cases $\xi\approx \sqrt{|\mathcal{I}|}$. We will define $\xi$ in Sec. \ref{sec:lazy-analysis}.


Obviously, in CFR with lazy update, we should balance the trade-off between running time and regret by selecting a suitable segmentation procedure. What surprising is that an extremely simple segmentation rule turned out to make a dramatic improvement. In our final algorithm, Lazy-CFR as shown in Alg. \ref{alg:adapt}, we simply update the strategy on an infoset by RM or Hedge if the cumulative reach probability on it becomes larger than a threshold. 
 And we further show that Lazy-CFR only needs to update the strategies on $O(\sqrt{|\mathcal{I}|})$ infosets in each round which is significantly smaller than $O(|\mathcal{I}|)$ in CFR, while the regret of Lazy-CFR can be controlled in $O(D\sqrt{\xi T\log A})$. Accordingly, Lazy-CFR requires running time $O(\frac{\sqrt{|\mathcal{I}|}\xi D^2\log A}{\epsilon^2})$ to compute an $\epsilon$-Nash Equilibrium, whilst the vanilla CFR needs $O(\frac{|\mathcal{I}|\xi D \log A}{\epsilon^2})$ running time. So that we accelerate CFR by a factor $O(\sqrt{|\mathcal{I}|}/D)$, which is a dramatic improvement in large scale games, since $D$ (the depth of the game tree) is usually in the order of $O(\log |\mathcal{I}|)$. 

We then analyze the regret lower bound. We show that no algorithm can achieve a regret lower than $\Omega(\sqrt{\xi T \log A})$ by constructing an explicit adversary. This means that the regrets of both CFR and Lazy-CFR are near-optimal within a factor of $O(\sqrt{D})$ and $O(D)$ respectively. 

We empirically evaluate our algorithm on the standard benchmark, Leduc Hold'em \citep{brown2015regret}. We compare with the vanilla CFR, MC-CFR \citep{lanctot2009monte}, and CFR+ \citep{bowling2017heads}. It is noteworthy that the same idea of Lazy-CFR can also be applied to CFR+, and we name the resulted algorithm Lazy-CFR+. 
 The analysis on Lazy-CFR can be directly applied to Lazy-CFR+. 
Experiments show that Lazy-CFR and Lazy-CFR+ dramatically improve the convergence rates of CFR and CFR+ in practice, respectively. 

The rest of this paper is organized as follows. Sec. \ref{sec:pre} reviews some useful preliminary knowledge of this work. In Sec. \ref{sec:lazy-analysis}, we present the idea of lazy update with the analysis, and our algorithm is presented in Sec. \ref{sec:alg}. After that, we present the regret lower bound in Sec. \ref{sec:lowerbound}. And then we discuss some related work in Sec. \ref{sec:related-work}. Finally, we show our experimental results in Sec. \ref{sec:exp}. 

\section{Notations and Preliminaries}\label{sec:pre}
We first introduce the notations and definitions of extensive games and TEGIs. Then we introduce an online learning concept of regret minimization. After that, we discuss the connection between TEGIs and regret minimization. This connection triggered the powerful algorithm, CFR. Finally, we finish this section by discussing the details of CFR.

\subsection{Extensive games}
Extensive games (see \citep{osborne1994course} page 200 for a formal definition.) compactly model the decision-making problems with sequential interactions among multiple agents. An extensive game can be represented by a game tree $H$ of histories, where a history is a  sequence of actions in the past. Suppose that there are $N$ players participating in an extensive game and let $c$ denote the chance player which is usually used to model the uncertainty in the environment. Let $[N]:=\{1,\cdots, N\}$. A player function $P$ is a mapping from  $H$ to $[N]\cup \{c\}$ such  that $P(h)$ is the player who takes an action after $h$.
And each player $i\in [N]$ receives a reward $u^i(h)\in[-1,1]$ at a terminal history $h$.

Let $\mathcal{A}(h)$ denote the set of valid actions of $P(h)$ after $h$, that is, $\forall a\in \mathcal{A}(h)$,  $(h,a)\in H$. Let $A=\max_{h}|\mathcal{A}(h)|$. A strategy of player $i$ is a function $\sigma^i$ which assigns $h$ a distribution over $\mathcal{A}(h)$ if $P(h)=i$. A strategy profile $\sigma$ consists of the strategy for each player, i.e., $\sigma^1,\cdots, \sigma^N$. We will use $\sigma^{-i}$ to refer to all the strategies in $\sigma$ except $\sigma^i$. And we use the pair $(\sigma^i, \sigma^{-i})$ to denote the full strategy profile. In games with imperfect information, actions of other players are partially observable to a player $i\in [N]$. So for player $i$, some different histories may not be distinguishable. Thus, the game tree can be partitioned into disjoint information sets (infoset). Let $\mathcal{I}^i$ denote the collection of player $i$'s infosets, we have that two histories $h,h'\in I \in \mathcal{I}^i$ are not distinguishable to player $i$. Thus, $\sigma^i$ should assign the same distribution over actions to all histories in an infoset $I\in \mathcal{I}^i$. With a little abuse of notations, we let $\sigma^i(I)$ denote the strategy of player $i$ on infoset $I\in \mathcal{I}^i$. 

Moreover, let $\pi_\sigma(h)$ denote the probability of arriving at a history $h$ if the players take actions according to strategy $\sigma$. Obviously, we can decompose $\pi_\sigma(h)$ into the product of each player's contribution, that is, $\pi_\sigma(h)=\prod_{[N]\cup \{c\}}\pi_{\sigma}^i(h)$. Similarly, we can define $\pi_\sigma(I)=\sum_{h\in I}\pi_{\sigma}(h)$ as the probability of arriving at an infoset $I$ and $\pi^i_\sigma(I)$ denote the corresponding contribution of player $i$.  Let $\pi_{\sigma}^{-i}(h)$ and  $\pi_{\sigma}^{-i}(I)$ denote the product of the contributions on arriving at $h$ and $I$, respectively, of all players except player $i$. 

In game theory, the \emph{solution} of a game is often referred to a \textbf{\emph{Nash equilibrium (NE)}} \citep{osborne1994course}. With a little abuse of notations, let $u^i(\sigma)$ denote the expectated reward of player $i$ if all players take actions according to $\sigma$. An NE is a strategy profile $\sigma^*$, in which every $\sigma^{*,i}$ is optimal if given $\sigma^{*,-i}$, that is, $\forall i \in [N]$, $u^i(\sigma^*)=\max_{\sigma^i} u^i((\sigma^i, \sigma^{*,-i}))$.

In this paper, we  concern on computing an approximation of an NE, 
 namely an $\mathbf\epsilon$\textbf{-NE} \citep{nisan2007algorithmic}, since computing an $\epsilon$-NE is usually faster in running time. An $\epsilon$-NE is a strategy profile $\sigma$ such that:
$$
    \forall i\in[N], u^i(\sigma)\geq \max_{\sigma'^{,i}}u^i((\sigma'^{,i}, \sigma^{-i})) - \epsilon.
$$

With the above notations, a two-player zero-sum extensive game with imperfect information (TEGI) is an extensive game with $N=2$ and $u^1(h)+u^2(h)=0$ for a terminal history $h$. And the $\epsilon$-NE in a TEGI can be efficiently computed by regret minimization, see later in this section.

\subsection{Regret minimization}
Now we introduce \textit{regret}, a core concept in online learning \citep{cesa2006prediction}. 
Many powerful online learning algorithms can be framed as minimizing some kinds of regret, therefore known as regret minimization algorithms. 
Generally, the regret is defined as follows:
 
\begin{Defn}[Regret]
 Consider the case where a player takes actions repeatedly. At each round, the player selects an action $w_t\in \Sigma$,\footnote{The valid action set $\Sigma$ is generalized, that is, the element in $\Sigma$ can be anything, e.g., a distribution or a vector in a Euclidean space.} where $\Sigma$ is the set of valid actions. At the same time, the environment \footnote{The environment may be an adversary in online learning.} selects a reward function $f_t$. Then, the overall reward of the player is $\sum_{t=1}^T f_t(w_t)$, and the regret is defined as:
 $$
    R_T=\max_{w'\in \Sigma}\sum_{t=1}^{T}f_t(w') - \sum_{t=1}^{T}f_t(w_t).
 $$
 \end{Defn}

One of the most famous 
example of online learning is \emph{online linear optimization} (OLO) in which $f_t$ is a linear function. If $\Sigma$ is the set of distributions over some discrete set  
, an OLO can be solved by standard regret minimization algorithms, e.g., regret matching (RM) \citep{blackwell1956analog,abernethy2011blackwell} or Hedge \citep{freund1997decision}. 

CFR employs RM or Hedge as a sub-procedure, so we summarize OLO, RM and Hedge as follows:
\begin{Defn}[Online linear optimization (OLO), regret matching (RM) and Hedge]\label{defn:rm}
Consider the online learning problem with linear rewards. In each round $t$, an agent plays a mixed strategy $w_t\in \Delta(\mathcal{A})$, where $\Delta(\mathcal{A})$ is the set of probabilities over the set $\mathcal{A}$, while an adversary selects a vector $c_t \in \mathbb{R}^{|\mathcal{A}|}$. The reward of the agent at this round is $\langle w_t, c_t\rangle$ where $\langle\cdot,\cdot \rangle$ denotes the operator of inner product. The goal of the agent is to maximize the cumulative reward which is equivalent to minimizing the following regret:
\begin{align*}
	R^{olo}_{T} = \max_{w\in \Delta(\mathcal{A})}\sum_{t=1}^{T}\langle w, c_t\rangle - \sum_{t=1}^{T} \langle w_t, c_t\rangle.
\end{align*}

Let $R^{olo}_{T,+}(a)=\max(0, \sum_{t=1}^{T}c_t(a) - \sum_{t=1}^{T} \langle w_t, c_t\rangle)$, RM picks $w_t$ as follows:
\begin{equation}
  w_{t+1}(a)=\begin{cases}
    \frac{R^{olo}_{t,+}(a)}{\sum_{a'}R^{olo}_{t,+}(a')}, &\max_{a'}R^{olo}_{t,+}(a')>0 .\\
    \frac{1}{|\mathcal{A}|}, & \text{otherwise}.
  \end{cases}
\end{equation}
According to the result in \citep{blackwell1956analog}, RM enjoys the following regret bound:\footnote{\citep{blackwell1956analog}'s result implicitly indicates  this regret bound. However, to the best of our knowledge, there's no existing work providing a regret bound in this form with a detailed proof. So we prove it in the Appendix \ref{app:proof-rm-upper}.}
\begin{align}\label{eq:rm_bound}
	R^{olo}_{T}\leq O\left(\sqrt{\sum_{t=1}^{T}\|c_t\|_2^2}\right).
\end{align}
Let $s_t(a)=\exp(\sum_{t'=1}^t c_{t'}(a))$, Hedge picks $w_t(a)=s_t(a)/(\sum_{a'}s_t(a')).$ 
According to \citep{freund1997decision}, Hedge enjoys the following regret bound:
\begin{align}\label{eq:rm_hedge}
    R^{olo}_T \leq O\left(\sqrt{\log |\mathcal{A}|\sum_{t=1}^{T}\max_{a\in \mathcal{A}}c_t^2(a)}\right).
\end{align}
\end{Defn}

\subsection{Counterfactual regret minimization (CFR)}
CFR is developed on a connection between $\epsilon$-NE and regret minimization. This connection is naturally established by considering repeatedly playing a TEGI as an online learning problem. It is worthy to note that there are two online learning problems in a TEGI, one for each player. 

Suppose player $i$ takes $\sigma^i_t$ at time step $t$ and let $\sigma_t=(\sigma_t^1, \sigma_t^2)$. Consider the online learning problem for player $i$ by setting $w_t:=\sigma^i_t$ and $f^i_t(\sigma^i):=u^i((\sigma^i, \sigma^{-i}_t))$.  The regret for player $i$ is $R_T^i:=\max_{\sigma}R_T^i(\sigma)$ where $R_T^i(\sigma):= \sum_{t=1}^{T}u^i((\sigma^i, \sigma_t^{-i})) - \sum_{t=1}^{T}u^i((\sigma_t^i, \sigma_t^{-i}))$.

Furthermore, define the time-averaged strategy, $\bar{\sigma}^i_T$, as follows:
$$
    \bar{\sigma}^i_T(I) = \frac{\sum_{t}\pi_{\sigma_t}^i(I)\sigma_{t}^i(I)}{\sum_{t}\pi_{\sigma_t}^i(I)}.
$$
It is well-known that \citep{nisan2007algorithmic}:
\begin{Lem}\label{lem:connection}
    If $\frac{1}{T}R_T^i\leq \epsilon/2$ for $i=1,2$, then $(\bar{\sigma}^1_{T}, \bar{\sigma}^2_T)$ is an $\epsilon$-NE.
\end{Lem}
However, it is hard to directly apply regret minimization  algorithms to TEGIs, since the reward function $u$ is non-convex with respect to $\sigma$. 
One approach is that as in \citep{gordon2007no}, we first transform a TEGI to a normal-form game, and then apply the Lagragian-Hedge algorithm \citep{gordon2005no}.
However, this approach is time-consuming since the dimension of the corresponding normal-form game is exponential to $|\mathcal{I}|$. 
To address this problem, \citet{zinkevich2008regret} propose a novel decomposition of the regret $R_{T}^i$ into the summation of immediate regrets as \footnote{\citet{zinkevich2008regret} directly upper bounded $R_T^i$     by the counterfactual regret, i.e., Eq. \eqref{eq:cfr-upper-bound}, and   omitted the derivation of Eq. \eqref{eq:regret-overall}. So we present the derivation of Eq. \eqref{eq:regret-overall} in Appendix \ref{app:proof-reg-upper}.}:
\begin{align}\label{eq:regret-overall}
 & R_T^{i}(\sigma) \nonumber\\
=& \sum_{t} \sum_{I\in \mathcal{I}^i, P(I)=i}\pi_{\sigma}^i(I)\pi_{\sigma_t}^{-i}(I) (u^i(\sigma_t|_{I\rightarrow \sigma(I)}, I) - u^i(\sigma_t, I))
\end{align}

where $\sigma|_{I\rightarrow \sigma'(I)}$ denotes the strategy generated by modifying $\sigma(I)$ to $\sigma'(I)$  and $u^i(\sigma, I)$ denote the reward of player $i$ conditioned on arriving at the infoset $I$ if the strategy $\sigma$ is executed.

Further, \citet{zinkevich2008regret} upper bound Eq. (\ref{eq:regret-overall}) by the \textbf{counterfactual regret}:
\begin{align}
 &R_T^{i}(\sigma) \nonumber\\
\leq& \sum_{I\in \mathcal{I}^i, P(I)=i} \underbrace{\left( \sum_{t}\pi_{\sigma_t}^{-i}(I) (u^i(\sigma_t|_{I\rightarrow \sigma(I)}, I) - u^i(\sigma_t, I))\right)}_{\text{\large The OLO for each infoset} }\label{eq:cfr-upper-bound}
\end{align}
 For convenience, we call  $\pi_{\sigma_t}^{-i}(I) u^i(\sigma_t|_{I\rightarrow a}, I)$ the \textbf{counterfactual reward} of action $a$ at round $t$. 

Notice that Eq. (\ref{eq:cfr-upper-bound}) essentially decomposes the regret minimization of a TEGI into $O(|\mathcal{I}|)$ OLOs. So that, in each round, we can apply RM  directly to each individual OLO to minimize the counterfactual regret. And the original regret $\max_\sigma R^i_T(\sigma)$ is also minimized since the counterfactual regret is an upper bound. So that,  with Eq. \eqref{eq:rm_bound}, Eq.  \eqref{eq:cfr-upper-bound} and the fact that the norm of a conterfactual reward vector is at most $O(\sqrt{A})$, we can upper bound the counterfactual regret by $O(|\mathcal{I}|\sqrt{AT})$. However, we have to traverse the whole game tree, which is very time-consuming in large scale games.

In the sequel, we are going to show that updating the strategy on every infoset is not indispensable. Intuitively, this is because the regret is determined by the norm of the vector of counterfactual reward on each node (see Eq. (\ref{eq:rm_bound})). And on most nodes, the corresponding norm is very small, since $\pi_{\sigma_t}^{-i}$ is a probability. 

\section{Lazy update and regret upper bound}\label{sec:lazy-analysis}
In this section, we present the idea of lazy update. We first discuss lazy update in the context of OLO. And then we leverage the idea of lazy update to extensive games. After that, we provide our analysis on the regret bound of CFR with lazy update in Sec. \ref{sec:lazy-extensive}. 
Our analysis is novel since it reveals the correlation among  immediate regrets and encodes the structure of the game tree explicitly. The regret bound is presented in our main theorem, Thm \ref{thm:main-upp-reg}. Furthermore, Thm \ref{thm:main-upp-reg} can also be used to analyze the regret bound of CFR. Thus, by applying Thm \ref{thm:main-upp-reg}, we refined the regret bound of the vanilla CFR. 
 
 The idea of lazy update and its analysis do not depend on the choice of OLO solver. So that in our demonstration of lazy update, we assume OLOs are solved by RM.  And similarly, we only present the detailed proof with RM as the OLO solver. Proof of the variant employing Hedge as the OLO solver can be simply derived by substituting Eq. \eqref{eq:rm_bound} with Eq.\eqref{eq:rm_hedge} that can achieve an $O(\sqrt{\log A})$ upper bound on $A$ rather than $O(\sqrt{A})$.
\subsection{Lazy update for OLOs}\label{sec:lazy-olo}
We now introduce lazy update for OLOs in Defn. \ref{defn:rm}, see Fig. \ref{fig:lazy-illustration} for an illustration.  
We call an online learning algorithm for OLOs as a lazy update algorithm if:
\begin{itemize}
    \item It divides time steps $[T]$ into $n$ disjoint  subsets with consecutive elements, that is,  $\{t_{i}, t_{i}+1,\cdots, t_{i+1}-1\}_{i=1}^{n}$ where $1=t_1< t_2\cdots < t_{n+1} = T+1$. For convenience, we call these subsets as segments. 
    \item It updates $w_t$ at time steps $t=t_{i}$  for some $i$ and keeps $w_t$ the same within each segment. That is, the OLO with $T$ steps collapses into a new OLO with $n$ steps. And we have $c'_j = \sum_{t=t_j}^{t_{j+1}-1}c_t$ where $c_t$ is the vector selected by the adversary in the original OLO at time step $t$ and $c'_j$ is the vector selected by the adversary in the collapsed OLO at time step $j$.
\end{itemize}
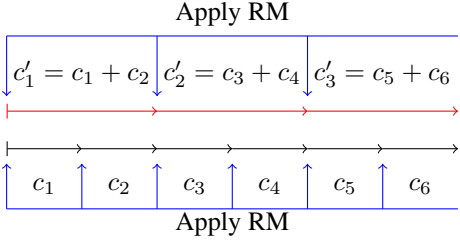
\begin{figure}
    \centering
    \begin{tikzpicture}
        \draw [|->] (0, 1)--(1, 1); \node at (0.5, 0.5) {$c_1$};
        \draw [->] (1, 1)--(2, 1);  \node at (1.5, 0.5) {$c_2$};
        \draw [->] (2, 1)--(3, 1);  \node at (2.5, 0.5) {$c_3$};
        \draw [->] (3, 1)--(4, 1); \node at (3.5, 0.5) {$c_4$};
        \draw [->] (4, 1)--(5, 1); \node at (4.5, 0.5) {$c_5$};
        \draw [->] (5, 1)--(6, 1); \node at (5.5, 0.5) {$c_6$};
        \draw [blue] (0,0.2)--(6,0.2);
        \node at (3, 0) {Apply RM};
        \draw [blue, ->] (0, 0.2)--(0,0.8);
        \draw [blue, ->] (1, 0.2)--(1,0.8);
        \draw [blue, ->] (2, 0.2)--(2,0.8);
        \draw [blue, ->] (3, 0.2)--(3,0.8);
        \draw [blue, ->] (4, 0.2)--(4,0.8);
        \draw [blue, ->] (5, 0.2)--(5,0.8);
        
        \draw [red, |->] (0, 1.5)--(2, 1.5);
        \node at (1, 2) {$c'_1=c_1+c_2$};
        \draw [red, ->] (2, 1.5)--(4, 1.5);
        \node at (3, 2) {$c'_2=c_3+c_4$};
        \draw [red, ->] (4, 1.5)--(6, 1.5);
        \node at (5, 2) {$c'_3=c_5+c_6$};
        \draw [blue] (0, 2.5) -- (6, 2.5);
        \node at (3, 2.8) {Apply RM};
        \draw [blue, ->] (0, 2.5) -- (0, 1.7);
        \draw [blue, ->] (2, 2.5) -- (2, 1.7);
        \draw [blue, ->] (4, 2.5) -- (4, 1.7);
    \end{tikzpicture}
    \caption{An illustration on RM with lazy update for OLOs.  On the bottom is the standard RM; on the top is the RM with lazy update. The lengths of time in the original OLO and the collapsed OLO are $6$ and $3$ respectively. Suppose $\|c_1\|_2=\|c_3\|_2=\|c_5\|_2=1$ and $\|c_2\|_2=\|c_4\|_2=\|c_6\|_2=0.01$. Then the regrets are almost the same, since $\sum_{t=1}^{6}\|c_t\|_2^2=3.0003$ and $\sum_{t=1}^3\|c'_t\|_2^2\leq 3\times 1.01^2\approx 3.03$.}
    \label{fig:lazy-illustration}
\end{figure}

According to Eq. \eqref{eq:rm_bound}, the regret of RM with lazy update is bounded by:
\begin{align}
	O\left(\sqrt{\sum_{i=1}^{n}  \left\|c'_i:=\sum_{t=t_{i}}^{t_{i+1}-1}c_t\right\|_2^2 }\right)\label{eq:regret-lazy}
\end{align}
RM with lazy update does not need to update the strategy at each round. And if the segmentation is reasonable, that is, $\sum_{i=1}^{n}\|c'_i\|^2\approx \sum_{j=1}^{T}\|c_j\|^2$, then  the regrets of the lazy update RM and the vanilla RM are similar in amount.

It is noteworthy that in OLOs, the running time of lazy update RM is still $O(AT)$ which is the same as applying RM directly, where $A$ is the dimension of $c_t$. This is because we have to compute $\sum_{t=t_i+1}^{t_{i+1}}c_t$ which is time-consuming. Fortunately, this problem can be addressed in TEGIs, see Sec. \ref{sec:alg} for how to overcome it by exploiting the structure of the game tree.

\subsection{Lazy update for TEGIs}\label{sec:lazy-extensive}
 We now extend the idea of lazy update to TEGIs. 
According to Eq. (\ref{eq:cfr-upper-bound}), the regret minimization procedure can be divided into $|\mathcal{I}|$   OLOs, one for each infoset. 
For convenience, for each infoset $I\in \mathcal{I}^i$, we divide the time steps $[T]$ into $n(I)$ segments $\{t_j(I), \cdots, t_{j+1}(I)-1\}_{j=1}^{n(I)}$ where $1=t_1(I)<t_2(I)\cdots <t_{n(I)+1}(I)=T+1$. Let $r_j(I,a) = \sum_{t=t_{j}(I)}^{t_{j+1}(I)-1}\pi^{-i}_{\sigma_t}(I)u^i(\sigma_t|_{I\rightarrow a}, I)$ denote the summation of the counterfactual rewards over a segment. And let $r_j(I)=[r_j(I,a)]_{a\in \mathcal{A}(I)}$ denote the vector consisting of $r_j(I, a), a\in \mathcal{A}(I)$. Similar to lazy update for OLOs,  we only update the strategy on infoset $I$ at $t_j(I)$ according to RM. Let $\sigma'_j(I)$ denote the strategy after the $j$-th update on infoset $I$, that is,  $\sigma_t(I):=\sigma'_{j}(I)$ for $t\in \{t_{j}(I),t_{j}(I)+1, \cdots, t_{j+1}(I)-1\}$. According to Eq. (\ref{eq:regret-lazy}), the regret of applying RM to the collapsed OLO on infoset $I$ can be bounded as: 
\begin{align*}
	R^{lazy}_T(I):=\max_{\sigma\in \Sigma(I)}\sum_{j=1}^{n(I)}\langle \sigma - \sigma'_j(I),  r_{j}(I) \rangle
    \leq  \sqrt{\sum_{j=1}^{n(I)}\|r_j(I)\|^2}
\end{align*}

Now we analyze the regret upper bound  on the overall regret of a TEGI, i.e., Eq. (\ref{eq:regret-overall}), of the above lazy update algorithm. The upper bound is presented in our main theorem, Thm. \ref{thm:main-upp-reg}. 

\begin{Thm}\label{thm:main-upp-reg}
The regret of CFR with lazy update can be bounded  as follows:
	\begin{eqnarray}
		 R_T^i(\sigma)&=O\left(
		 \sqrt{
		 \left(\sum_{I',t}\pi^i_{\sigma}(I')\pi^{-i}_{\sigma_t}(I')\right)\eta(\sigma)}
		 \right)\label{eq:tight-bound-raw}\\
		&\leq \sqrt{DT \eta(\sigma) }\label{eq:tight-bound}
	\end{eqnarray}
	where $\eta(\sigma):=\sum_{I\in \mathcal{I}^i} \pi_{\sigma^i}^{i}(I)\frac{\sum_{j=1}^{n(I)}f(r_j(I))}{\sum_{t=1}^{T}\pi_{\sigma_t}^{-i}(I)}$ where $f(r_j(I))=\|r_j(I)\|^2$ for CFR with RM and $f(r_j(I))=\max_{a}|r_j(I)(a)|^2\log A $ for CFR with Hedge.
    \begin{proof}[Sketch of proof] We defer the proof to  Appendix \ref{app:proof-reg-upper}. Our analysis is also on the immediate regrets as in \citep{zinkevich2008regret} and the improvements are on following two aspects:

    	     First, instead of providing an upper bound on the counterfactual regret in Eq. (\ref{eq:cfr-upper-bound}), we directly analyze the bound of the original regret Eq. (\ref{eq:regret-overall}). This makes us to be able to analyze $\pi^i_{\sigma}(I)$'s effect on $R_T^i(\sigma)$. To see how this improves the regret bound intuitively, consider the case that 
    	    $R^{lazy}_T(I)$ is large on an infoset $I$. Though $R^{lazy}_T(I)$ is large, it makes $R_T^i(\sigma)$ increasing dramatically only if $\pi_\sigma^i(I)$ is also large. This is because the immediate regret of infoset $I$  is $\pi_{\sigma}^i(I)R^{lazy}_T(I)$ and $R_T^i(\sigma)$ is the summation of immediate regrets. Moreover, it is impossible that $\pi_{\sigma}^i(I)$ is very large on all infosets, since $\pi_{\sigma}^i(I)$ is the probability of arriving at $I$ contributed by player $i$
    	    . So that the immediate regrets cannot be very large at the same time.

    	    Second, we upper bound the regret by quantities ($\sum_{I,t}\pi^{i}_{\sigma}(I')\pi_{\sigma_t}^{-i}(I')$ and $\eta (\sigma)$ in Eq.  (\ref{eq:tight-bound-raw})) which can reflect the structure of the underlying game tree. So that we can give a more detailed analysis on these quantities, which leads to a tighter regret bound.
    \end{proof}
\end{Thm}


 According to Thm \ref{thm:main-upp-reg}, we can bound the regret by bounding $\max_{\sigma}\eta(\sigma)$. In the sequel of this paper, we upper bound $\eta(\sigma)$ using the following inequalities:

 \begin{Lem}
    Let $\mathcal{G}(\sigma) = \xi^i \max_{I,j}(\sum_{t=t_j(I)+1}^{t_{j+1}}\pi^{-i}_{\sigma_t}(I))$, for CFR with RM, we have: 
    \begin{align}
        \eta(\sigma)
        \leq A\mathcal{G}(\sigma)\label{eq:ratio-upper-bound}
    \end{align} 
    for CFR with Hedge, we have:
    \begin{align*}
        \eta(\sigma)
        \leq \mathcal{G}(\sigma)\log A
    \end{align*}
where $\xi^i = \max_{\sigma}\sum_{I\in \mathcal{I}^i, P(I)=i}\pi_{\sigma}^i(I)$, which is significantly smaller than $|\mathcal{I}|$ since $\pi_\sigma^i(I)$ is a probability. For convenience, let $\xi = \max(\xi^1, \xi^2)$.
 \end{Lem}

\begin{proof}[Derivation of Eq. (\ref{eq:ratio-upper-bound})]
With straight-forward computations, we have:
    \begin{align}
        \eta(\sigma)&\leq A \sum_{I\in \mathcal{I}^i}\pi^i_{\sigma^i}(I)\frac{\sum_{j=1}^{n(I)}\|\sum_{t=t_j(I)+1}^{t_{j+1}(I)}\pi^{-i}_{\sigma_t}(I)\|^2}{\sum_t \pi^{-i}_{\sigma_t}(I)}\nonumber\\
        &\leq A\sum_{I\in \mathcal{I}^i}\pi^i_{\sigma^i}(I) \max_{j}(\sum_{t=t_j(I)+1}^{t_{j+1}}\pi^{-i}_{\sigma_t}(I))\nonumber\\
        &\leq A\max_{I,j}\left(\sum_{t=t_j(I)+1}^{t_{j+1}}\pi^{-i}_{\sigma_t}(I)\right)   \left(\sum_{I\in \mathcal{I}^i} \pi_{\sigma}^i(I)\right) \nonumber\\
        &\leq O(\xi^i A) \max_{I,j}\left(\sum_{t=t_j(I)+1}^{t_{j+1}}\pi^{-i}_{\sigma_t}(I)\right)\nonumber
    \end{align}
\end{proof}

\emph{\textbf{A tighter regret bound of CFR}}: 
It is easy to see that the vanilla CFR is a special case of lazy update, in which $t_j(I)+1=t_{j+1}(I)$ for every $j, I$. 
So we can apply Thm \ref{thm:main-upp-reg} and Eq. (\ref{eq:ratio-upper-bound}) to CFR directly, which leads to a tighter regret bound as in the following lemma. 

\begin{Lem}\label{lem:tighter}
	With RM, the regret of the vanilla CFR is bounded by $O(\sqrt{\xi DAT})$ . With Hedge, the regret of the vanilla is $O(\sqrt{\xi D T\log A})$.
    \begin{proof}
    	We only need to bound $\max_{\sigma}\eta(\sigma)$ and then insert it into Thm \ref{thm:main-upp-reg}. By directly applying Eq. (\ref{eq:ratio-upper-bound}) and the fact that $\pi_{\sigma_t}^{-i}(I)\leq 1$, we have $\eta(\sigma)\leq O(A\xi^i)$.
    \end{proof}
\end{Lem}


 \section{Lazy-CFR}\label{sec:alg}
 
\begin{algorithm}[t]
\caption{Lazy-CFR}
\label{alg:adapt}  
\begin{algorithmic}[1]
\STATE A two-player zero-sum extensive game.
\STATE Initialize the reward vector $r(I)$ for all $I\in \mathcal{I}^i$
\WHILE {$t<T$}
    \FORALL{$i\in\{1, 2\}$}
	    \STATE $Q=\{I_r\}$ where $I_r$ is the root of the infosets tree.
        \WHILE{$Q$ is not empty.}
    	    \STATE Pop $I$ from $Q$.
            \STATE Update the strategy on $I$ via RM or Hedge.
            \STATE For $I'\in succ(I)$, if $m_t(I')\geq 1$, push $I'$ into $Q$.
        \ENDWHILE
    \ENDFOR 
    \FORALL{$I$}
        \STATE Update the reward vector on $I$ if any player's strategy on infosets below $I$ has been modified.
    \ENDFOR 
\ENDWHILE
\end{algorithmic}
\end{algorithm}
 In this section, we discuss how to design an efficient variant of CFR with the framework of lazy update. 
 
 Intuitively, an efficient $\epsilon$-NE solver for TEGIs, which is based on minimizing the regret of the OLO on each infoset,  should satisfy the following two conditions. The first one is to prevent the overall regret from growing too fast. And according to Thm \ref{thm:main-upp-reg}, we only need to make $\eta(\sigma')$ to be small for all $\sigma'$. Furthermore, this can be done by making $\|\sum_{t=t_j(I)+1}^{t_{j+1}}\pi^{-i}_{\sigma_t}(I)\|$ small for all $I, j$ in the framework of lazy update. The second condition is to update as small number of infosets as possible during a round, which is equivalent to make $\sum_{I\in \mathcal{I}^i}n(I)$ small. 

  The key to balance the tradeoff between $n(I)$ and $\max_{j}\|\sum_{t=t_{j}(I)+1}^{t_{j+1}(I)}\pi_{\sigma_t}^{-i}(I)\|$ is that  $\max_{\sigma}\sum_{I\in \mathcal{I}^i}\pi_{\sigma}^{-i}(I)$ is significantly smaller than $|\mathcal{I}^i|$. 
  To make a clean upper bound on $\max_{\sigma}\sum_{I\in \mathcal{I}^i}\pi_{\sigma}^{-i}(I)$, we make the following mild assumption: 
\begin{Asmp}\label{asmp:asmp1}
   1), If $P(h)=i$, then $P((h,a))\neq i$; 2)The tree of infosets for each player is a full $A$-ary tree. 
 \end{Asmp}
  
  Assumption \ref{asmp:asmp1} naturally leads to the following corollary:
  \begin{Cor}\label{cor:direct-extension}
  If a TEGI satisfies Assumption \ref{asmp:asmp1}, then $\forall \sigma, \sum_{I\in \mathcal{I}^i}\pi^{-i}_{\sigma}(I)=O(\sqrt{|\mathcal{I}^i|})$.
  \begin{proof}
    For simplicity, we denote $I_{h}^{i}\in \mathcal{I}^i$ as the corresponding infoset of $h$ for player $i$. Notice that, $\forall h\in H, P(h)=i$, we have $\forall a, \pi_{\sigma}^{-i}(I_{h}^{i}) = \pi_{\sigma}^{-i}(I_{h, a }^{-i})$ and $\sum_{a}\pi_{\sigma}^{-i}(I_{h, a, a^\prime}^{i}) = \pi_{\sigma}^{-i}(I_{h, a}^{-i})$. Thus, with Assumption \ref{asmp:asmp1}, $\sum_{a, a^\prime} \pi_{\sigma}^{-i}(I_{h, a, a^\prime}^{i}) = A\pi_{\sigma}^{-i}(I_{h}^{i})$, which means the accumulated reaching probability increases $O(A)$-times every two layer while the nodes (i.e. infosets) increases $O(A^2)$-times every two layer. Thus $\forall \sigma, \sum_{I\in \mathcal{I}^i}\pi^{-i}_{\sigma}(I)=O(\sqrt{|\mathcal{I}^i|})$.
  \end{proof}
  \end{Cor}
 
 With the above analysis, our algorithm is pretty simple:  at time step $t$, let $\tau_t(I)$ denote the last time step we update the strategy on infoset $I$ before $t$. Let $m_t(I):=\sum_{\tau=\tau_t(I)+1}^t \pi_{\sigma_{\tau}}^{-i}(I)$ denote the summation of  probabilities of arriving at $I$ after $\tau_t(I)$, which is contributed by all players except $i$. Let $subt(I)$ denote the subtree \footnote{Here, the tree is composed of infosets rather than histories.} rooted at infoset $I$. Let $succ(I)$ denote a subset of $subt(I)$ such that $\forall I'\in succ(I), P(I')=i$ and $\forall I''\in subt(I)$, if $I''$ is an ancient of $I'\in succ(I)$, then $P(I'')\neq i$ or $I''=I$. We simply update the strategies on infosets recursively as follows: 
 after updating the strategy on infoset $I$, we keep on updating the strategies on the infosets from $succ(I)$ with $m_t(I)\geq 1$. 
 We summarize our algorithm in Alg. \ref{alg:adapt}.
 
 Now we analyze Alg. \ref{alg:adapt}. 
  To give a clean theoretical result, we further make the following assumption:
    \begin{Asmp}\label{asmp:asmp2}
       Every infoset in the tree of infosets is corresponding to $n$ nodes in the game tree.
    \end{Asmp}
  
   It is noteworthy that Alg. \ref{alg:adapt} is still valid and efficient without both Assumption \ref{asmp:asmp1} and \ref{asmp:asmp2}.

  Now we present our theoretical results on the regret and time complexity of  Alg. \ref{alg:adapt} in  Lem \ref{lem:lazy-regret} and Lem \ref{lem:running-time} respectively. 
  \begin{Lem}\label{lem:lazy-regret}
      If the underlying game satisfies Assumption \ref{asmp:asmp1}, then, with RM , the regret of Alg. \ref{alg:adapt} is bounded by $O(D\sqrt{\xi AT})$; with Hedge, the regret can be bounded by   $O(D\sqrt{\xi T\log A})$.
      
      \begin{proof}
          According to Thm \ref{thm:main-upp-reg} and Eq. (\ref{eq:ratio-upper-bound}), we only need to bound $\max_t m_t(I)$. Below, we prove that $m_t(I)\leq d(I)$ where $d(I)$ is the depth of infoset $I$ in the game tree.
          
          We exploit mathematical induction to prove $m_t(I)\leq d(I)$. If it holds for $d(I)\leq d$. Consider $d+1$, it is obvious that at the last time step at which its parent was updated, there is at most $1$ cumulated probability at infoset $I$, thus $m_t(I)\leq m_t(pa(I))+1<d(pa(I))+1\leq d(I)$ where $I\in succ(pa(I))$.
      \end{proof}
      
  \end{Lem}
  
 	\begin{Lem}\label{lem:running-time}
 	The time complexity of Alg. \ref{alg:adapt} is $O(\# nodes_t)$ at round $t$, where $\# nodes_t$ is the number of nodes in the tree of histories which are touched by Alg. \ref{alg:adapt} during round $t$. 
 	    More specifically, if the underlying game satisfies both Assumption \ref{asmp:asmp1} and \ref{asmp:asmp2}, then the time complexity of Alg. \ref{alg:adapt} in each round is $O(\sqrt{|\mathcal{I}^i|} n)$ on average.
 	    \begin{proof}
 	        There is a little engineering involved to prove the first statement. We defer the details of the engineering into Appendix \ref{app:implement-details}. 
 	        
 	        The second statement is proved as follows. According to Corollary \ref{cor:direct-extension}, on average, there are at most $O(\sqrt{|\mathcal{I}|})$ infosets which satisfy $m_t(I)>1$, so that there are $O(n\sqrt{|\mathcal{I}|})$ nodes touched in each round. 
 	    \end{proof}
 	\end{Lem}
 
 According to Lem \ref{lem:lazy-regret}, \ref{lem:running-time} and \ref{lem:connection}, the regret is about $O(\sqrt{D})$ times larger than the regret of CFR, whilst the running time is about $O(\sqrt{|\mathcal{I}}|)$ times faster than CFR per round. 
 Thus, according to Lem \ref{lem:connection}, \ref{lem:tighter} and with a little algebra, we know that Alg. \ref{alg:adapt} is $O(\sqrt{|\mathcal{I}|}/D)$ times faster than the vanilla CFR to achieve the same approximation error, since the vanilla CFR has to traverse the whole game tree in each round. The improvement is significant in large scale TEGIs.
 
 \textbf{\emph{Lazy-CFR+}}: We can directly apply the idea of lazy update to CFR+ \citep{bowling2017heads}, which is a novel variant of CFR. CFR+ uses a  different regret minimization algorithm instead of RM. \citet{tammelin2015solving} prove that the running time of CFR+ is at most in the same order as CFR, but in practice CFR+ outperformes CFR. To get Lazy-CFR+, we only need to replace RM by RM+ in Alg. \ref{alg:adapt} and use the method of computing time-averaged strategy as in \citep{bowling2017heads}. We empirically evaluate Lazy-CFR+ in Sec. \ref{sec:exp}.
 
 \begin{figure*} [t]
\centering
\subfigure[Leduc-5]{\label{fig:leduc5}\includegraphics[width=0.3\textwidth]{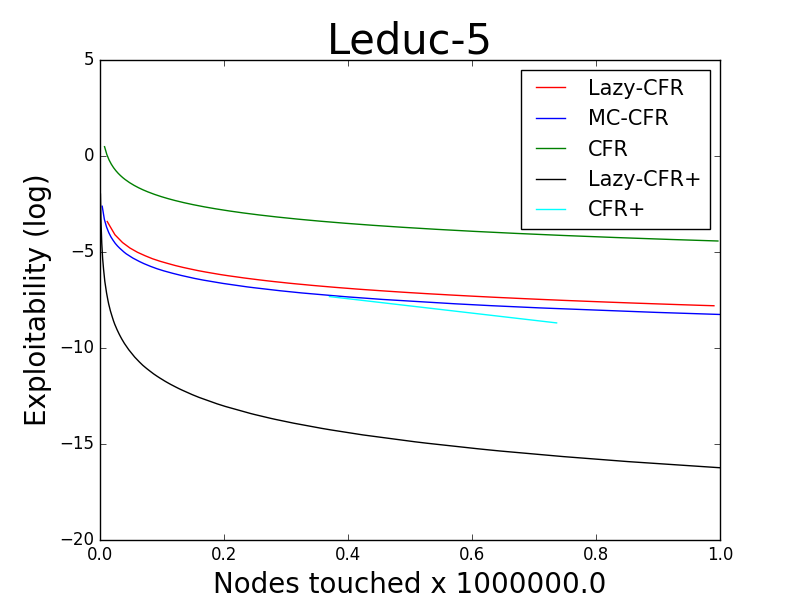}
}
\quad
\subfigure[Leduc-10]{\label{fig:leduc10}\includegraphics[width=0.3\textwidth]{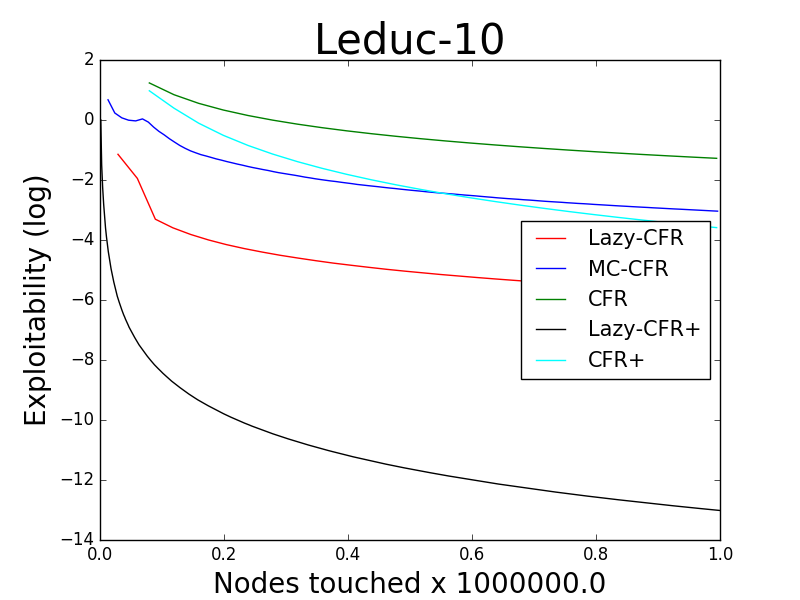}
}
\quad
\subfigure[Leduc-15]{\label{fig:leduc15}\includegraphics[width=0.3\textwidth]{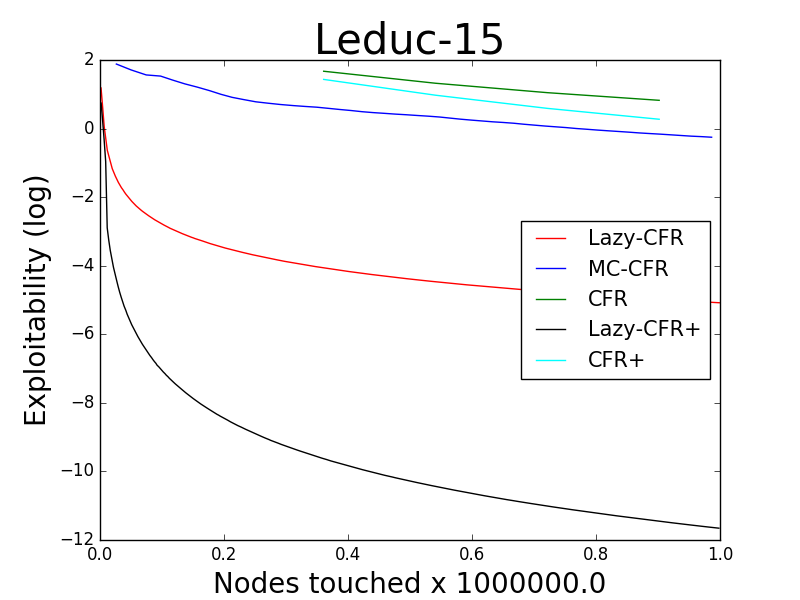}
}
\caption{Convergence for Lazy-CFR, Lazy-CFR+, MC-CFR, CFR and CFR+ on the Leduc Hold'em.}
\label{fig:leduc}
\end{figure*}
 
\section{Regret lower bound}\label{sec:lowerbound}
 In this section, we analyze the worst case lower bound on the regret. 
 We consider the standard adversarial setting in online learning, in which an adversary selects $\sigma_t^{-i}$ and a reward function $u^i_t: Z\rightarrow [0,1]$ where $Z$ is the set of terminal nodes in the infoset tree of player $i$. 
 
 Thus, to get a lower bound, we need to explicitly construct $\sigma_{t}^{-i}$ and $u^i_t$.  For convenience, 
 let $-i$ denote the players besides $i$, let $\zeta=\{\widehat{\sigma}^i: \widehat{\sigma}^i=\argmax_{\sigma^i}\sum_{I\in \mathcal{I}^i, P(I)=i}\pi_{\sigma^i}^i(I)\}.$ Let $M:=\{I\in \mathcal{I}^i: \exists \sigma^i\in \zeta, \pi^i_{\sigma^i}(I)>0 \}.$ It can be shown that $M$ forms a subtree of $\mathcal{I}^i$. 
  Our construction on $\sigma_t^{-i}$ and $u^i_t$ can be divided into two stages. 

\begin{enumerate}
    \item For $I\in Z, I\notin M$, $u_t^i(I)=0$ for all $t$. This enforces player $i$ take actions on $M$, otherwise, it will receive reward $0$. 
    \item For $I\in M, P(I)\neq i$, we first generate $a(I)\sim \rm{Multinomial}(\frac{\mathbf{1}}{A}),$ and then set $\sigma_t^{-i}(I)(a(I))=1$, and $\sigma_t^{-i}(I)(a)=0$ for $a\neq a(I)$. Intuitively, this step separates $R_T^i$ into $O(\xi^i)$ isolated OLOs, each of which is with dimension $A$ and would be repeated for about $O(T/\xi^i)$ rounds, since only one of them will be triggered on according to our construction on $\sigma_t^{-i}$. Thus, combined with the lower bound proved by \citep{cesa2006prediction}, each OLO incurs a regret of about $\Omega(\sqrt{T/\xi^i\log A})$. Then we can prove that $\max_{\pi_{\sigma_t}^{-i}, u_t} R_{T}^i \geq \xi^i \Omega(\sqrt{ T/\xi^i\log A})$ as in Thm \ref{thm:lowerbound}.
\end{enumerate}

\begin{Thm}\label{thm:lowerbound}
    For any algorithm, we have:
    \begin{align}
        \sup_{T, A}\max_{\pi_{\sigma_t}^{-i}, u_t} R_{T}^i \geq \Omega(\sqrt{\xi^i T\log A})
    \end{align}
    \begin{proof}
        See Appendix \ref{app:lowerbound} for a formal proof.
    \end{proof}
\end{Thm}

By comparing the regret lower bound in Thm \ref{thm:lowerbound} and the regret  upper bounds of CFR and Lazy-CFR as in Lem \ref{lem:tighter} and Lem \ref{lem:lazy-regret}, we can see that both CFR and Lazy-CFR have near-optimal regret bounds within a factor of $O(\sqrt{D})$ and $O(D)$ respectively.

 \section{Related work}\label{sec:related-work}
There are several variants of CFR which attempt to avoid traversing the whole game tree at each round. Monte-Carlo based CFR  (MC-CFR) \citep{lanctot2009monte}, also known as CFR with partial pruning,  uses Monte-Carlo sampling to avoid updating the strategy on infosets with small probability of arriving at. %
 Pruning-based variants \citep{brown2016reduced,brown2015regret} skip the branches of the game tree if they do not affect the regret, but their performance can deteriorate to the vanilla CFR on the worst case. 

In order to solve a large scale extensive game, there are several techniques used except CFR, we give a brief summary on these techniques. \citep{brown2017safe} proposed a technique on sub-game solving, which makes us to be able to solve the NE on a  subtree.
 Another useful technique is abstraction. 
 Abstraction \citep{gilpin2007lossless} reduces the computation complexity by solving an abstracted game, which is much smaller than the original game. 
There are two main kinds of abstractions, lossless abstraction and lossy abstraction. 
Lossless abstraction algorithms~\citep{gilpin2007lossless}  ensures that each equilibrium in the abstracted game is also an equilibrium in the original game.
For Poker games, it is able to reduce the size of the game by one-to-two orders of magnitude.
Lossy abstraction algorithms~\citep{kroer2014extensive, sandholm2015abstraction} 
create smaller, coarser game, in the cost of a decrease in the solution quality. 
Both of above two kinds of abstractions can be used to reduce the number of actions or the number of information sets. 


\section{Experiment}\label{sec:exp}

In this section, we empirically compare our algorithm with existing methods. We compare Lazy-CFR and Lazy-CFR+ with CFR, MC-CFR, and CFR+. 

In our experiments, we use RM as the OLO solver as it is commonly  used in practice and we do not use any heuristic pruning in CFR, CFR+, Lazy-CFR and Lazy-CFR+.  

Experiments are conducted on variants of the Leduc hold'em \citep{brown2015regret}, which is a common benchmark in imperfect-information game solving. Leduc hold'em is a simplifed version of the Texas hold'em. In Leduc hold'em, there is a deck consists of 6 cards, two Jack, two Queen and two King. There are two dealt rounds in the game. In the first round, each player receives a single private card. In the second round, a single public card is revealed. A bet round takes place after each dealt round, and player $1$ goes first. In our experiments, the bet-maximum varies in $5, 10$ and $15$. 

As discussed in Lem \ref{lem:running-time}, Alg. \ref{alg:adapt} uses $O(1)$ running time for each touched node, which is the same as in the vanilla CFR, CFR+ and MC-CFR. Thus, we compare the number of touched nodes of these algorithms, since nodes touched is independent with hardware and implementation. 

We measure the exploitability of these algorithms. The exploitability of a strategy $(\sigma^1, \sigma^2)$ can be interpreted as the approximation error to the Nash equilibrium. The exploitability is defined as $\max_{\sigma'^{,1}}u^1((\sigma'^{,1}, \sigma^2)) + \max_{\sigma'^{,2}}u^2((\sigma^1, \sigma'^{, 2}))$. 

Results are presented in Fig. \ref{fig:leduc}. The performance of  Lazy-CFR is slightly worse than MC-CFR and CFR+ on Leduc-5. But as the size of the game grows, the performance of Lazy-CFR outperforms all baselines. And Lazy-CFR+ consistently outperforms other algorithms. Thus, empirical results show that our method, lazy update, is a powerful technique to accelerate regret minimization algorithms for TEGIs. More specifically, on our largest experiment, Leduc-15, with over $1.3\times 10^5$ infosets, Lazy-CFR converges over 200 times faster than CFR, and Lazy-CFR+ converges over 500 times faster than CFR+.

\bibliography{example_paper}
\bibliographystyle{icml2019}

\appendix
\newpage
\onecolumn

\section{The details of implementation}\label{app:implement-details}

In this section, we discuss how to efficiently implement the idea of lazy update in TEGIs. 

The challenge is that if we are going to update the strategy on an infoset, we need to compute the summation of counterfatual rewards over a  segment. If we compute the summation directly, then lazy update enjoys no improvement compared with the vanilla CFR. Fortunately, this problem can be addressed in TEGIs. The key observation is that the reward on an infoset changes if and only if the strategies on some infosets in the subtree changed. 

More specifically, suppose we are going to compute the cumulative counterfactual reward on a history $s$ over the segment $[T_1, T_2]$ and $[T_1, T_2]$ is divided into $[T_1=:t_1,t_2 - 1], [t_2,t_3-1],...,[t_{n-1}, t_n-1=T_2]$ where $u^i(s, t')$ keeps the same in $[t_i, t_{i+1}-1]$. Then we have:

\begin{align*}
     \sum_{t=T_1}^{T_2} \pi^{\sigma_t}_{-i}(s) u^i(s; \sigma_t)  = \sum_{j=1}^{n-1} u^i(s; \sigma_{t_j}) \sum_{t'=t_j}^{t_{j+1}-1} \pi_{-i}(\sigma_t) 
\end{align*}

Obviously, by exploiting the tree structure and with elementary data structures, both $u^i(s; \sigma_{t_j})$ and  $\sum_{t'=t_j}^{t_{j+1}-1} \pi_{-i}(\sigma_t)$ can be computed in running time $O(1)$. Thus, we only need to analyze the number of segments.

We exploit the following property to bound the number of segments: 
\begin{Prop}
In Lazy-CFR, if the strategy on an infoset $I$ is updated in a round and $I\in succ(I')$, then $I'$ is also updated in this round.
\end{Prop}

So that, suppose there are $N_t$ infosets' strategies updated in round $t$, there are at most $O(N_t)$ new segments. 

For convenience, let $J(h, a)$ denote the history of taking action $a$ after $h$ and $J(I, a)$ denote the infoset of taking action $a$ after $h$. Suppose there are $n$ kinds of privation infomation for each player. For $I\in \mathcal{I}^i$, let $\mathcal{H}(I, j)$ denote the underlying history if the opponent's private information   is the $j$-th one. We elaborate Lazy-CFR in a more detailed pseudo-code, Alg. \ref{alg:detail}.

\begin{algorithm}[t]
\caption{A detailed implementation of Lazy-CFR}
\label{alg:detail}  
\begin{algorithmic}[1]
\STATE A two-player zero-sum extensive game.
\STATE Randomly initialize $\sigma$.
\STATE $\forall h\in H$, compute the counterfactual value $cfv^i(h)=u^i(h|\sigma), flag(h)=-1$.
\STATE $\forall{I}\in \mathcal{I}, \bar{\sigma}(I)=0, s(I)=0, \forall a\in \mathcal{A}(I), CFV(I, a)=0$.
\WHILE {$t<T$}
    \FORALL{$i\in\{1, 2\}$}
        \STATE UPDATE1$(I_r^i, i, [1.0, \cdots, 1.0]).$
    \ENDFOR
        \STATE UPDATE2$(h_r, t, 0.0)$ where $h_r$ denotes the root of the history tree.
\ENDWHILE
\STATE \textbf{RETURN} $\bar{\sigma}$.
\end{algorithmic}
\end{algorithm}

\begin{algorithm}[t]
\caption{UPDATE1($I,i, p=[p(1),\cdots, p(n)], t$)}
\label{alg:upd_by_rm}  
\begin{algorithmic}[1]
\STATE $\forall j\in [n], m_1(I, j) = m_1(I, j) + p(j), m_2(I, j) = m_2(I, j) + p(j).$

\IF{$\sum_{j\in [n]} m_1(I, j)\geq 1$}
    \IF{$P(I)=i$}
        \FORALL{$a\in \mathcal{A}(I)$}
            \STATE Create $p'$ with $p'(j)=m_2(I, j) + p(j)$
            \STATE UPDATE1($J(I, a), i, p, t$).
        \ENDFOR
        \STATE $\forall j\in [n],$ updflag$(\mathcal{H}(I, j), t).$
        \FORALL{$a\in \mathcal{A}(I), j\in[n]$}
            \STATE $h=\mathcal{H}(I, j).$
            \STATE $CFV(I, a)=CFV(I, a) + m_2(I, j)\times cfv^i(J(\mathcal{H}(I, j), a))$
            \STATE $m_1(I, j)=m_2(I, j) = 0.$
        \ENDFOR
        \STATE $\bar{\sigma}(I) = (s(I)\bar{\sigma}(I) + \sigma^i(I)\sum_{j\in [n]} m_1(I, j))/(s(I) + \sum_{j\in [n]}m_1(I, j))$, $s(I)=s(I) + \sum_{j\in[n]}m_1(I, j)$
        \STATE $\sigma^i(I)=$RM$(CFV(I))$.
    \ELSE 
        \FORALL{$a\in\mathcal{A}(I)$}
            \STATE $p'=[0.0, \cdots, 0.0]$
            \FORALL{$j\in[n]$} 
                \STATE $p'(j)=(p(j) + m_2(I, j))\times \sigma(\mathcal{H}(I, j), a)$
            \ENDFOR
            \STATE UPDATE1$(J(I, a), i, p', t)$
        \ENDFOR
        \STATE $\forall j, m_1(I, j)=m_2(I, j)=0$
    \ENDIF
\ENDIF
\STATE \textbf{RETURN} $\bar{\sigma}$.
\end{algorithmic}
\end{algorithm}

\begin{algorithm}[t]
\caption{updflag$(h, t)$}  
\begin{algorithmic}[1]
\IF{$h$ is not the root of history tree and $flag(h)\neq t$}
    \STATE $flag(h)=t.$
    \STATE $updflag(pa(h))$ where $pa(h)$ is the parent of $h.$
\ENDIF
\STATE \textbf{RETURN} $\bar{\sigma}$.
\end{algorithmic}
\end{algorithm}

\begin{algorithm}[t]
\caption{UPDATE2$(h, t, p)$}  
\begin{algorithmic}[1]
\IF{$flag(h)=t$}
    \STATE $i=P(h), I=\mathcal{K}(h, i).$
    \STATE Consider $j$ satisfies $\mathcal{H}(I, j)=h.$ 
    \FORALL{$a\in\mathcal{A}(h)$}
        \STATE  $CFV(I, a) = CFV(I, a)+cfv^i(J(h, a))\times (m_2(I, j) + p).$
        \STATE UPDATE2$(J(h, a), t, (p + m_2(I, j)) \sigma^i(I, a))$.
    \ENDFOR 
    \STATE update $cfv^i(h).$
    \STATE $m_2(I, j)=0.$
\ENDIF
\end{algorithmic}
\end{algorithm}

\section{Proof of the regret bound in Eq.  \eqref{eq:rm_bound}}\label{app:proof-rm-upper}
Regret matching can be viewed as an instance of online mirror descent (OMD) algorithm \citep{shalev2012online}. We first introduce the Fenchel's duality for some kinds of regularizer $R$:
\begin{align*}
R^*(x) = \sup_{w}( \langle x, w \rangle - R(w))
\end{align*}
We can get the derivative:
\begin{align*}
    \nabla R^*(x) = \mathop{\arg\max}_w(\langle x, w \rangle - R(w))
\end{align*}
OMD (see Algorithm \ref{alg:OMD}) is a natural implementation of the Follow-the-Regularized-Leader framework (FoReL) \citep{shalev2012online}. In FoReL, $w_t$ is selected by 
\begin{align*}
    \forall t, w_t = \mathop{\arg\max}_{w} \sum_{i=1}^{t-1} f_i(w) + R(w)
    \footnotemark
\end{align*}
\footnotetext{In \citep{shalev2012online} the environment provides the loss functions $f_t$, thus they need to minimize the loss and apply mirror descent each round. In our setting, we need to maximize the reward, so we instead apply mirror ascent here. We still use the name online mirror descent to avoid the potential ambiguity.}
\begin{algorithm}
\caption{Online Mirror Descent}
\label{alg:OMD}
\begin{algorithmic}[1]
\STATE Initialize: $x_1=\mathbf{0}$
\FOR {$t=1, 2, \cdots$}
    \STATE predict $w_t=\nabla R^{*}(x_t)$
    \STATE update $x_{t+1} = x_t + z_t$, where $z_t \in \partial f_t(w_t)$
\ENDFOR
\end{algorithmic}
\end{algorithm}
It is known that if $R$ is $\frac{1}{\eta}$-strongly-convex, then OMD enjoys the regret bound:
\begin{align*}
    {\rm Regret}_{T}(u) \leq R(u) - \min_wR(w) + \eta \sum_{t=1}^T \|z_t\|_{*}^2
\end{align*}
where $\|\cdot\|_{*}$ is the dual norm.
\begin{proof}
In regret matching, $w\in \Delta(\mathcal{A})$. We choose $R(w) = \frac{\|w\|_{2}^2}{2\eta}$, and it's easy to find that when $w_i = \frac{[x]_{i, +}}{\sum_i [x]_{i, +}}$, $\langle x,w\rangle - R(w)$ get its supremum $\sum_i \frac{[x]_{i, +}^2}{2\eta}$, where $[x]_{i, +} = \max(x_i, 0)$. 

To show that regret matching is one kind of OMD, We first prove that $c_t - \langle w_t, c_t \rangle \mathbf{1} \in \partial_{w_t} \langle w_t, c_t\rangle$, which is equivalent to prove that $\forall u\in \Delta(\mathcal{A})$,
\begin{align*}
    \sum_{t=1}^{T}\langle u - w_t, c_t \rangle \leq & \sum_{t=1}^{T}\langle u-w_t, c_t - \langle w_t, c_t \rangle \mathbf{1}  \rangle
\end{align*}
Notice that $\langle u - w_t, \mathbf{1}\rangle = 0$ as $u, w_t \in \Delta(\mathcal{A})$, so the above inequality holds $\forall t$.

Thus, regret matching enjoys the regret bound:
\begin{align*}
    {\rm Regret}_{T}(\Delta(\mathcal{A})) \leq & \max_{w \in \Delta(\mathcal{A})} R(w) - \min_{w \in \Delta(\mathcal{A})} R(w) + \eta \sum_{t=1}^T \|c_t\|_{2}^2\\
    \leq & \frac{1}{2\eta} + \eta \sum_{t=1}^T \|c_t\|_{2}^2\\
\end{align*}
Let $\eta = 1/\sqrt{\sum_{t=1}^T \|c_t\|_{2}^2}$, we prove \eqref{eq:rm_bound}
\end{proof}
\section{Proof of Thm \ref{thm:main-upp-reg}}\label{app:proof-reg-upper}
We first prove Eq. \eqref{eq:regret-overall}.
\begin{proof}
Let $J^i(I,a)$ denote the infoset of player $i$ after $P(I)$ takes $a$ at $I$. Without loss of generality, we assume that $P(I^i_r)=i$ where $I^i_r$ is the root of the tree of player $i$'s infoset. Then we have:
\begin{align*}
    R_T^i(\sigma) &= \sum_{t=1}^T u^i((\sigma^i, \sigma_t^{-i})) - \sum_{t=1}^T u^i((\sigma_t^i, \sigma_t^{-i}))\\
    &=\sum_{t=1}^T \left(u^i((\sigma^i, \sigma_t^{-i})) - u^i((\sigma_t^i|_{I_r^i\rightarrow \sigma^i(I)}, \sigma_t^{-i}))\right) +   \sum_{t=1}^T\left( u^i((\sigma_t^i|_{I_r^i\rightarrow \sigma^i(I)}, \sigma_t^{-i})) -  u^i((\sigma_t^i, \sigma_t^{-i}))\right)\\
    &= \sum_{t=1}^T\left( u^i((\sigma^i|_{I_r^i\rightarrow \sigma_t^i(I)}, \sigma_t^{-i})) -  u^i((\sigma_t^i, \sigma_t^{-i}))\right) \\
    &+ \sum_{t=1}^T\sum_{a\in A(I_r^i)} \sigma^i(I_r^i, a)\left(u^i((\sigma^i, \sigma_t^{-i}), J^i(I^i_r, a)) - u^i((\sigma_t^i|_{I_r^i\rightarrow \sigma^i(I)}, \sigma_t^{-i}), J^i(I^i_r, a))\right)\\
    &\cdots \text{(recursively use this decomposition to the second term.)}\\
    &=\sum_{t} \sum_{I\in \mathcal{I}^i}\pi_{\sigma}^i(I)\pi_{\sigma_t}^{-i}(I) (u^i(\sigma_t|_{I\rightarrow \sigma(I)}, I) - u^i(\sigma_t, I))
\end{align*}
\end{proof}
Now we prove Theorem \ref{thm:main-upp-reg}.
\begin{proof}
With Eq. \eqref{eq:regret-overall} and the regret bound of RM, we have
	\begin{align*}
	\frac{1}{T}R_T^{i}(\sigma) &= \frac{1}{T}\sum_{t} \sum_{I\in \mathcal{I}^i}\pi_{\sigma}^i(I)\pi_{\sigma_t}^{-i}(I) (u^i(\sigma_t|_{I\rightarrow \sigma(I)}, I) - u^i(\sigma_t, I))\\
        &\leq \sum_{I\in \mathcal{I}^i}\pi^i_{\sigma}(I)O\left(\sqrt{\sum_{j=1}^{n(I)}\|r_j(I)\|^2 /T^2}\right)
	\end{align*}
	
	And then apply Jensen's inequality and with some calculations, we have 
    \begin{align*}
    	\frac{1}{T} R_T^{i}(\sigma) & 
        \leq \sum_{I\in \mathcal{I}^i}\pi^i_{\sigma}(I)O\left(\sqrt{\sum_{j=1}^{n(I)}\|r_j(I)\|^2/T^2}\right)\\
        &= \frac{\sum_{I,t}\pi^i_{\sigma}(I')\pi^{-i}_{\sigma_t}(I')}{\sum_{I',t}\pi^i_{\sigma}(I')\pi^{-i}_{\sigma_t}(I')} \sum_{I\in \mathcal{I}^i}\pi^i_{\sigma}(I)O\left(\sqrt{\sum_{j=1}^{n(I)}\|r_j(I)\|^2/T^2}\right)\\
        &= \sum_{I',t}\pi^i_{\sigma}(I')\pi^{-i}_{\sigma_t}(I') \sum_{I\in \mathcal{I}^i}\frac{\pi^i_{\sigma}(I)}{\sum_{I',t}\pi^i_{\sigma}(I')\pi^{-i}_{\sigma_t}(I')}O\left(\sqrt{\sum_{j=1}^{n(I)}\|r_j(I)\|^2/T^2}\right)\\
        &=\sum_{I',t}\pi^i_{\sigma}(I')\pi^{-i}_{\sigma_t}(I') \sum_{I\in \mathcal{I'}^i}\frac{\pi^i_{\sigma}(I)\sum_t \pi_{\sigma_t}^{-i}(I)}{\sum_{I',t}\pi^i_{\sigma}(I')\pi^{-i}_{\sigma_t}(I')}O\left(\sqrt{\frac{\sum_{j=1}^{n(I)}\|r_j(I)\|^2}{\left(\sum_t \pi_{\sigma_t}^{-i}(I)\right)^2T^2}}\right)\\
        &\leq \sum_{I',t}\pi^i_{\sigma}(I')\pi^{-i}_{\sigma_t}(I') O\left(\sqrt{\sum_{I\in \mathcal{I}^i}\frac{\pi^i_{\sigma}(I)\sum_t \pi_{\sigma_t}^{-i}(I)}{\sum_{I',t}\pi^i_{\sigma}(I')\pi^{-i}_{\sigma_t}(I')}\frac{\sum_{j=1}^{n(I)}\|r_j(I)\|^2}{\left(\sum_t \pi_{\sigma_t}^{-i}(I)\right)^2T^2}}\right)\\
        &=O\left(\sqrt{\left(\sum_{I',t}\pi^i_{\sigma}(I')\pi^{-i}_{\sigma_t}(I')\right)\sum_{I\in \mathcal{I}^i} \pi^i_{\sigma}(I)\frac{\sum_{j=1}^{n(I)}\|r_j(I)\|^2}{\sum_t \pi_{\sigma_t}^{-i}(I)T^2}}\right)\\
        &\leq O\left(\sqrt{\frac{D}{T}\sum_{I\in \mathcal{I}^i} \pi^i_{\sigma}(I)\frac{\sum_{j=1}^{n(I)}\|r_j(I)\|^2}{\sum_{t=1}^T \pi_{\sigma_t}^{-i}(I)}}\right)
    \end{align*}
    The last inequality utilizes the fact that  $\sum_{I',t}\pi^i_{\sigma}(I')\pi^{-i}_{\sigma_t}(I')\leq DT$. 
    Now we finished the proof.
    \end{proof}
\section{The lower bound analysis}\label{app:lowerbound}
Notice that the application of CFR does not depend on the game and adversary. So we here assume the adversary choose both its strategy profile and the utility each turn.



Let $J^i(I, a)$ denotes the infoset of player $i$ after $P(I)$ takes $a$ at $I$, $n_t(I)$ denotes the time player $i$ arrives at $I$ in the first $t$ round. We first show the following corollary:
\begin{Cor}
\label{cor:same_order}
    Let $\sigma^\prime_i := \mathop{\arg\max}_{\sigma}\sum_{I\in\mathcal{I}}\pi_{\sigma}^i(I)$ and $\xi_{\sigma^\prime_i, Z}^i := \sum_{I\in\mathcal{I}_Z} \pi_{\sigma^\prime_i}^i(I)$, we have that
    \begin{align}
        \xi^i = \Theta(\xi_{\sigma^\prime_i, Z}^i)
    \end{align}
\end{Cor}
Here we make an implicit assumption that $|\mathcal{A}(I)| \geq 2, \forall I\in\mathcal{I}^{P(I)}$. Otherwise we can merge these infosets as we have no choice but choose the only action, which contributes nothing to the regret. With this assumption, we can simply prove Corollary \ref{cor:same_order} with the same techniques using in the proof of Corollary \ref{cor:direct-extension}.

Now we prove Theorem \ref{thm:lowerbound}.
\begin{proof}

Without loss of generality, we assume $P(I)\neq i, \forall I\in Z$ where $Z$ is the set of terminal nodes (i.e. player $i$ takes the final action, otherwise, we can merge the adversary's final action selection into the utility design). Let $\mathcal{D}$ denotes the case satisfies Assumption \ref{asmp:asmp1} with $u_t(I)\sim \rm{Bernoulli}(0.5)$, $\forall I\in\mathcal{I}_{Z}$ and $\pi_{\sigma_t}^{-i}(I)\sim \rm{Multinomial}(\frac{\mathbf{1}}{A})$, $\forall I\in\mathcal{I}^i$. $a_{\sigma_t}^{-i}(I)$ denotes the action sample from  $\pi_{\sigma_t}^{-i}(I)$ at round $t$. 

Notice that in $\mathcal{D}$, for any $\sigma_i$, $\sum_{I\in\mathcal{I}} \pi_{\sigma_i}^i(I)$ is the same, and we can omit the max operation in the following proof. For general cases, we can set $u_t(I) = 0$ if $\pi_{\sigma^\prime_i}(I)=0, I\in Z$, then we back to this case. For notation simplicity, we assume the infoset tree is a full $A$-ary tree.

Let $\xi^i_{Z}= \sum_{I\in Z}\pi_{\sigma}^i(I)$. We prove that $\sup_{T, A}\mathbb{E}_{\mathcal{D}}R_T^{i} \geq \sqrt{\frac{\xi^i_{Z} T\log A}{2}}$. As \citet{auer1995gambling, freund1997decision, dani2008price} have shown, the lower bound for $K$-arm online linear optimization problem is $\sqrt{\frac{T\log K}{2}}$, which is consistent to our $\sqrt{\frac{\xi^i_Z T \log A}{2}}$ with $\xi^i_Z = 1$.

With a little abuse of notation, we use the term $R_T^i(I)$ and $\xi^i_Z(I)$ to represent the regret and the cumulative reaching probability of the leaf nodes rooted at $I$. 

We first show that if $P(I_r) \neq i$ and $\sup_{T, A}\mathbb{E}_{\mathcal{D}}R_{T}^i(J^i(I, a)) \geq \sqrt{\frac{\xi^i_{Z}(J^i(I, a))T\log A}{2}}, \forall a\in \mathcal{A}(I_r)$, then $\sup_{T, A}\mathbb{E}_{\mathcal{D}}R_T^{i} \geq \sqrt{\frac{\xi^i_{Z} T\log A}{2}}$. Notice that each subtree rooted at $J^i(I_r, a), a\in\mathcal{A}(I_r)$ have the same $\xi^i$, thus $\xi^i_Z(I_r) = A \xi^i_Z(J^i(I_r, a)), \forall a\in\mathcal{A}(I)$, and we have that



\begin{align*}
    \sup_{T, A}\mathbb{E}_{\mathcal{D}}R_T^i(I_r) = & \sup_{T, A}\mathbb{E}_{\mathcal{D}} \sum_{a\in \mathcal{A}(I_r)}\mathds{1}_{a_{\sigma_t}^{-i}(I_r) = a}  R_T^i(J^i(I_r, a))\\
    = & \sup_{T, A} \sum_{a\in\mathcal{A}(I_r)} \mathbb{E}_{\mathcal{D}}R_{n_T(I)}^i(J^i(I_r, a))\\
    = & \sum_{a\in\mathcal{A}(I_r)} \sup_{n_T(J^i(I, a)), A}\mathbb{E}_{\mathcal{D}}R_{n_T(J^i(I, a))}^i(J^i(I_r, a))\\
    = & A \sqrt{\frac{\xi^i_Z(J^i(I_r, a)) T \log A}{2A}}\\
    = & \sqrt{\frac{\xi^i_Z(I_r) T \log A}{2}}
\end{align*}
The third equality holds as each subtree rooted at $J^i(I_r, a), a\in\mathcal{A}(I_r)$ are independent.



Then we consider the case $P(I_r) = i$ with $\sup_{T, A}\mathbb{E}_{\mathcal{D}}R_{T}^i(J^i(I, a)) \geq \sqrt{\frac{\xi^i_{Z}(J^i(I, a))T\log A}{2}}, \forall a\in \mathcal{A}(I_r)$. We employ an additional notation $R_{T, imm}^i(I_r) = \max_{a\in \mathcal{A}(I_r)} \sum_{t=1}^T u_t(\sigma_t|_{I_r\to a}, I_r) - u_t(\sigma_t, I_r)$ to denote the immediate regret on the root node. It's obvious that $\mathbb{E}_{\mathcal{D}}R_{T, imm}^i(I_r)\geq 0$.


With Eq. \eqref{eq:regret-overall}, we can get:
\begin{align*}
    \sup_{T, A}\mathbb{E}_{\mathcal{D}}R_T^i(I_r) = & \sup_{T, A}\mathbb{E}_{\mathcal{D}} [R_{T, imm}^i(I_r) + \langle \sigma^i(I_r), [R_{n_T(J^i(I_r, a))}^i(J^i(I_r, a))]_{a\in\mathcal{A}(I_r)}\rangle] \\
    \geq & 0 + \langle \sigma^i(I_r), [\sup_{n_T(J^i(I, a)), A}\mathbb{E}_{\mathcal{D}}R_{n_T(J^i(I_r, a))}^i(J^i(I_r, a))]_{a\in\mathcal{A}(I_r)}\rangle\\
    \geq & \sum_{a\in\mathcal{A}(I_r)}\sqrt{\frac{\sigma^i(I_r, a)\xi^i(J^i(I_r, a))T\log A}{2}}\\
    \geq & \sqrt{\frac{\sum_{a\in \mathcal{A}_{I_r}}\sigma^i(I_r, a)\xi^i(J^i(I_r, a))\log T}{2}}\\
    = & \sqrt{\frac{\xi^i_Z(I_r) T \log A}{2}}
\end{align*}
Notice that the first inequality holds as given $\sigma^i(I_r)$, each subtree rooted at $J^i(I_r, a), a\in\mathcal{A}(I_r)$ are also independent.

Thus, with mathematical induction, we prove that $\sup_{T, A}\mathbb{E}_{\mathcal{D}}R_T^i\geq \sqrt{\frac{\xi^i_Z T \log A}{2}}$. With Corollary \ref{cor:same_order}, we can get the mini-max lower bound in Theorem \ref{thm:lowerbound}.
\end{proof} 
\end{document}